\documentclass{article}

\usepackage{arxiv}
\usepackage{graphicx}
\usepackage[utf8]{inputenc} 
\usepackage[T1]{fontenc}    
\usepackage{url}            
\usepackage{booktabs}       
\usepackage{amsfonts}       
\usepackage{nicefrac}       
\usepackage{microtype}      
\usepackage{lipsum}
\graphicspath{ {./images/} }
\usepackage[bookmarks=false]{hyperref}
\usepackage{tabularx}
\usepackage{amsmath}
\usepackage{algorithm}
\usepackage{algorithmic}
\usepackage{subfigure}
\usepackage{multirow}
\usepackage{amssymb}
\usepackage{booktabs}
\usepackage{amsthm}

\newtheorem{theorem}{Theorem}
\newtheorem{definition}{Definition}
\newtheorem{proposition}{Proposition}
\newtheorem{lemma}{Lemma}

\title{DWMD: Dimensional Weighted Orderwise Moment Discrepancy for Domain-specific Hidden Representation Matching}

\author{
 Rongzhe Wei \\
  School of Mathematics and Statistics\\
  Xi'an Jiaotong University\\
  Xi'an, China, 710049 \\
  \texttt{jessonwrz@163.com} \\
   \And
 Fa Zhang \\
  School of Computer Science and Technology\\
  Xi'an Jiaotong University\\
  Xi'an, China, 710049 \\
  \texttt{fazhang@stu.xjtu.edu.cn} \\
  \And
 Bo Dong \\
  School of Computer Science and Technology\\
  Xi'an Jiaotong University\\
  Xi'an, China, 710049 \\
  \texttt{dong.bo@mail.xjtu.edu.cn} \\
  \And
 Qinghua Zheng \\
  School of Computer Science and Technology\\
  Xi'an Jiaotong University\\
  Xi'an, China, 710049 \\
  \texttt{qhzheng@mail.xjtu.edu.cn} \\
 }

\begin{document}
\maketitle
\begin{abstract}
Knowledge transfer from a source domain to a different but semantically related target domain has long been an important topic in the context of unsupervised domain adaptation (UDA). A key challenge in this field is establishing a metric that can exactly measure the data distribution discrepancy between two homogeneous domains and adopt it in distribution alignment, especially in the matching of feature representations in the hidden activation space. Existing distribution matching approaches can be interpreted as failing to either explicitly orderwise align higher-order moments or satisfy the prerequisite of certain assumptions in practical uses. We propose a novel moment-based probability distribution metric termed dimensional weighted orderwise moment discrepancy (DWMD) for feature representation matching in the UDA scenario. Our metric function takes advantage of a series for high-order moment alignment, and we theoretically prove that our DWMD metric function is error-free, which means that it can strictly reflect the distribution differences between domains and is valid without any feature distribution assumption. In addition, since the discrepancies between probability distributions in each feature dimension are different, dimensional weighting is considered in our function. We further calculate the error bound of the empirical estimate of the DWMD metric in practical applications. Comprehensive experiments on benchmark datasets illustrate that our method yields state-of-the-art distribution metrics. 
\end{abstract}


\section{Introduction}
In the era of big data, preprocessing and labeling large quantities of collected data is cost-inefficient and consumes considerable human resources ~\cite{zellinger2017central}. These, in turn, limit the training of deep networks since the training stages require labeled data as prior knowledge for supervised learning. However, if we directly adopt the already available classifier models pretrained on labeled datasets and apply them to unlabeled and differently distributed data, the performance will dramatically degrade. Thus, leveraging the knowledge from one data distribution, namely, the source domain, with sufficiently labeled, model pretrained samples to establish a well-performed classifier model on a different but semantically related target data distribution with samples left unannotated is the fundamental goal of UDA~\cite{pan2009survey} ~\cite{zhuang2019comprehensive}. 

Previous shallow domain adaptation (DA) or transfer learning (TL) methods under unsupervised settings have shown that bridging the source and target domains through learning domain-invariant hidden feature representations is promising ~\cite{huang2007correcting} ~\cite{pan2010domain} ~\cite{gong2013connecting}. However, with the prevalence of deep networks in recent years, studies have shown that more and stronger transferable features in the hidden activation space can be obtained by embedding UDA in the pipeline of deep learning ~\cite{tzeng2014deep} ~\cite{long2015learning} ~\cite{sun2016deep} ~\cite{ganin2017domain}. In this case, a series of most common and successful practices develop a metric between deep networks' domain-specific hidden activations, and by minimizing this metric distance, reduce distribution differences between source and target domains to maximize the similarity of two domain hidden representations. 

\begin{figure*}[htb]
  \centering
  \includegraphics[width=\linewidth]{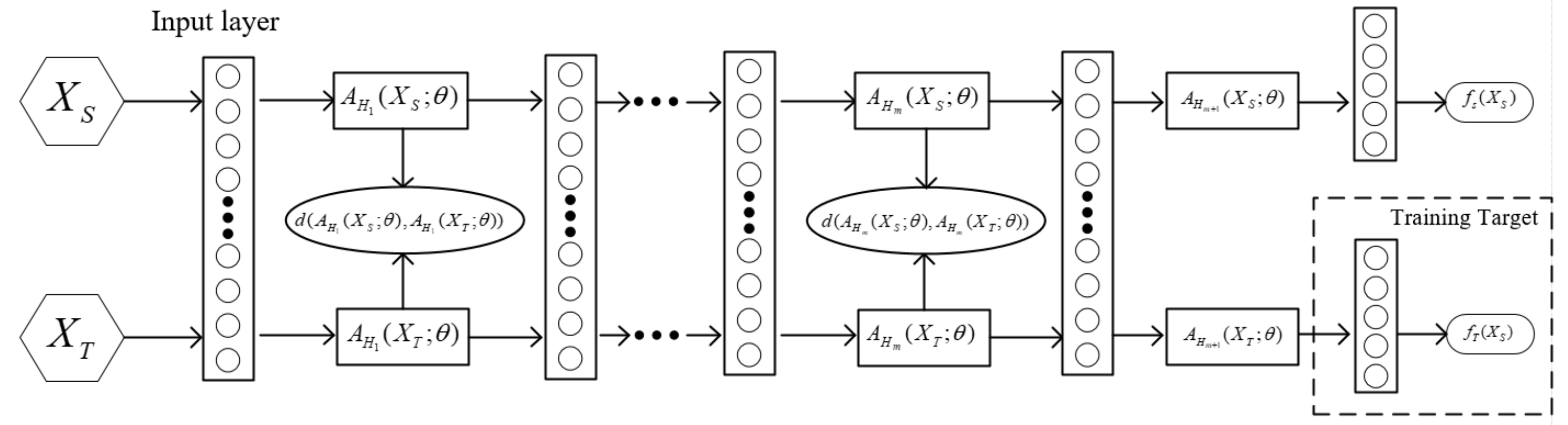}
  \caption{Schematic framework of multi-layer hidden representation matching in deep networks. The discrepancy of hidden activations of both source and target domains in each hidden layer are measured by domain regularizer $d$ whose gradient is added to the objective in the backpropagation training.}
  \label{fig:fig1}
\end{figure*}

Many outstanding measurements have been proposed, which can be divided into two categories: nonmoment-based and moment-based metrics. As a typical representative of nonmoment-based metrics, Proxy $\mathcal{A-}distance$ was proposed by Ben-David et al., defined by $\hat{d}_{\mathcal{A}} = 2 \left( 1 - 2 \epsilon \right)$. $\epsilon$ in the distance function is the generalization error of discriminating the samples from source and target domains. Later, the well-known domain adversarial neural networks (DANN) proposed in ~\cite{ganin2017domain} calculate the same $\epsilon$ value with a neural network classifier trained simultaneously with a feature extraction network by means of a gradient reversal layer (GRL).

Many of the widely used distribution difference metrics can be interpreted as matching statistical moments. The Kullback-Leibler (KL) divergence, proposed by Kullback \& Leibler (1951), was applied to the activation space in deep networks in ~\cite{zhuang2015supervised}, which can be viewed as the mean (first raw moment) matching of hidden representations. In addition, KL divergence, by its definition, is not a strictly defined distribution metric in mathematics since it violates triangle inequality and symmetry properties, and the two probability density functions used for the KL divergence calculation are sometimes unattainable in real scenarios. In 2006, a kernel-based distribution measurement termed the maximum mean discrepancy (MMD) was proposed in ~\cite{gretton2007kernel}, which has been widely applied in the field of DA. MMD can be simplified and computed using kernel trick, where the distance of instance mean values is calculated in a reproducing kernel Hilbert space (RKHS). ~\cite{gretton2012optimal} took advantage of multiple kernels and proposed a multikernel version of MMD, i.e., MK-MMD. Later, ~\cite{yan2017mind} proposed a weight version of MMD aiming to address the class weight bias issue. All these measurements are MMD-based, and many of them have been introduced in hidden representation matching tasks. However, according to the work of ~\cite{li2015generative}, most MMD-based methods with Gaussian kernels can be viewed as minimizing distances between weighted sums of all raw moments by using Taylor expansion, which does not take in the explicit high-order moment matching form.

In 2017, central moment discrepancy (CMD) for hidden representation matching tasks was proposed in ~\cite{zellinger2017central}. The CMD method is a direct high-order central moment alignment metric that yields state-of-the-art MMD-based approaches on nearly all benchmark datasets. The proposed CMD metric is only valid under the assumption that the source and target data are distributed within compact intervals, which are finite closed intervals in finite-dimensional Euclidean space. However, the prerequisite of this assumption can hardly be satisfied in real-world scenarios, since the values of hidden activation features in each dimension can extend to infinity under many activation functions, such as ReLU, Leaky ReLU, etc. Therefore, the assumption of this distribution metric has narrowed the selection of activation functions for hidden representation matching by only using bounded functions such as the sigmoid function, which can slow down the training speed in deep networks.

To address the aforementioned shortcomings, we propose a novel distribution discrepancy metric termed dimensional weighted orderwise moment discrepancy (DWMD). Our metric function, taken in the form of a series, explicitly orderwise aligns high-order moments. We theoretically prove that our function is strictly a metric by its mathematical definition, and our measurement is valid without any distribution assumption given, which can be applied to any activation functions in deep networks. Furthermore, we consider the degree of discrepancy in each dimension of hidden representations, and dimensional weighting is added to the total metric. If each term of the series can be considered a measure of the vertical difference between two domains with respect to each moment order, then the role of the dimensional weighted vector is to horizontally weight each feature dimension to reflect the discrepancy. In addition, the error bound of our metric when using its empirical estimate is given.

In experiments, we selected \emph{Office-31}, \emph{ImageCLEF-DA}, \emph{Office-10} as benchmark datasets, and comprehensive results demonstrate the superiority of our method compared with state-of-the-art metrics in hidden representation matching tasks.

The reminder of this paper is structured as follows. A brief overview of hidden representation matching and the corresponding notations are provided in section 2. In section 3, we propose our DWMD method included with theoretical analysis. Thereafter, we describe the experimental setup and results in section 4. Detailed discussions on specific topics with respect to the DWMD metric are illustrated in section 5. Section 6 presents the conclusion of our work.

\section{Hidden Representation Matching and Notations}
In this section, we present a brief overview of hidden representation matching and introduce the notations used in later sections.

Let $\mathcal{X} = \mathbb{R}^{d}$ and $\mathcal{Y} = \left[ 0, 1\right]^{|C|}$ denote the feature space and label space respectively. The feature space is d-dimensional, and $|C|$ is the cardinality of all categories $C$. In this paper, we mainly deal with the homogeneous UDA setting, where both source and target data share the same $\mathcal{X}$ and $\mathcal{Y}$. The source domain $\mathcal{D}_{S}$ and the target domain $\mathcal{D}_{T}$ are two different but semantically related distributions over $\mathcal{X} \times \mathcal{Y}$. The source domain samples $S = \left\{ X_S, Y_S \right\} = \left\{ \left( x_i, y_i \right) \right\}_{i = 1}^{n}$ are all annotated, while the target domain samples $T = \left\{ X_T \right\} = \left\{ x_j^{*}  \right\}_{j = 1}^{m}$ are left unlabeled. 

As is shown in Fig ~\ref{fig:fig1}, in the UDA scenario, the model classifier $f_S: \mathcal{X} \rightarrow \mathcal{Y}$ for source samples has already been trained with parameters $\theta_S$, and our main goal is to obtain a well-performing classifier $f_T: \mathcal{X} \rightarrow \mathcal{Y}$, namely, to obtain its parameters $\theta_T$ in parameter space $\Theta$, by minimizing its target risk $R_T \left( f \right) = Pr \left( f_T \left( x^{*} \right) \neq y^{*} \right)$, where $y^{*}$ is the actual label of sample $x^{*}$. 

For each hidden layer $\mathcal{H}_{i}$, both the source and target hidden activations are denoted as $A_{H_i} \left( X_S; \theta \right)$ and $A_{H_i} \left( X_T; \theta \right)$ with activation functions $g_{H_i}$. If $\mathcal{H}_{i}$ has $n_i$ hidden nodes, then $A_{H_i} \left( X_S; \theta \right)$, $A_{H_i} \left( X_T; \theta \right) \in \mathbb{R}^{n_i}$. In the hidden representation matching task, a distribution discrepancy metric or so-called a domain regularizer $d$ is in added. The domain regularizer $d: \left( \mathbb{R}^{n_i} \right)^{n} \times \left( \mathbb{R}^{n_i} \right)^{m} \rightarrow \mathbb{R}_{+}$ is used for the training of parameters $\theta_T$. Since there is no annotation for target data, the empirical estimation of source domain loss under parameters $\theta \in \Theta$ is utilized. Let $\mathbb{E} \left( l \left( X_S, Y_S, \theta \right)\right)$ define the empirical estimation and $l: \mathcal{X} \times \mathcal{Y} \times \Theta$ is the selected loss function, e.g. cross-entropy function. In addition, $\lambda$ is the penalty parameter of domain regularizer $d$.

The total objective for obtaining $\theta_T \in \Theta$ is to minimize:
$$\mathbb{E} \left( l \left( X_S, Y_S, \theta_T \right)\right) + \lambda \cdot d \left( A_{H} \left( X_S; \theta_T \right), A_{H} \left( X_T; \theta_T \right) \right)$$ 

In the network backpropagation training, the gradient is written as:
$$\nabla_{\theta_T} \mathbb{E} \left( l \left( X_S, Y_S, \theta_T \right)\right) + \lambda \cdot \nabla_{\theta_T} d$$

\section{Methodology}
In this section, we first introduce the definition of our new distribution distance metric termed DWMD.  Then, we present the calculation of the dimensional weighted vector ($\tau \left( X_S, X_T \right)$) in the metric, which measures the averaging discrepancy of feature spaces in each dimension between the source and target domains. The supporting theorems and propositions are given after, followed by the proof of our probability discrepancy measure. Finally, the domain regularizers for hidden activation matching are defined by the empirical estimate of our DWMD.
\subsection{Dimensional Weighted Orderwise Moment Discrepancy (DWMD)}
To realize accurate measurement of the feature distribution discrepancy between two similar but different domains, we take advantage of the moment generating function to construct high order moment alignment metric. This metric is an expansion of moment distance given in the form of series. For each order, the DWMD function is given as follows:
\begin{definition}
(DWMD metric) Considering an n-dimensional Euclidean space, $X_S$ and $X_T$ are independent and identically distributed random vectors from the source domain $D_S$ and target domain $D_T$ respectively. The dimensional weighted orderwise moment discrepancy (DWMD) function is defined by\\
\begin{equation*}
\mathcal{D}_{new} \left( \mathcal{D}_S || \mathcal{D}_T \right) = \sum_{n = 1}^{\infty} e^{\omega} \odot \frac{| \mathbb{E}X_{S}^{n} - \mathbb{E}X_{T}^{n} |^{\beta}}{C + | \mathbb{E}X_{S}^{n} - \mathbb{E}X_{T}^{n}|^{\beta}}
\end{equation*}
\begin{equation*}
\omega = -\frac{\psi n}{\tilde{\tau} \left( X_S, X_T \right)}
\end{equation*}
\end{definition}
 \quad \\
where $\mathbb{E}X^{n}$ is the $n^{th}$ raw moment of random variable $X$. Note that sometimes for convenience, we will use $n$ to denote the number of items in the DWMD function. Parameter $\beta \leq 1$. $C$ and $\psi$ are positive constants. We suggest $C \sim \mathcal{O}\left( \mathbb{E}X_S - \mathbb{E}X_T \right)$ for the reason that if $C$ is large or small enough then we will have:
$$\frac{| \mathbb{E}X_{S}^{n} - \mathbb{E}X_{T}^{n} |^{\beta}}{C + | \mathbb{E}X_{S}^{n} - \mathbb{E}X_{T}^{n}|^{\beta}} \rightarrow 0 \text{  or  }1$$
Thus, we recommend to just simply set $C = \tau^{1} \left( X_S, X_T \right)$, which can be obtained from dimensional weighted vector without extra calculation. $\tilde{\tau} \left( X_S, X_T \right)$ represents the dimensional weighted vector measuring the mean discrepancy of feature space in each dimension. Its detailed calculation is given in the next subsection.
\subsection{Dimensional Weighted Vector Calculation}
To evaluate the discrepancies between the source and target domains in each dimension of the feature space, we introduced a dimensional weighted vector $\tau \left( X_S, X_T \right) \in \mathbb{R}^{d}$, which plays a key role in horizontal discrepancy reflection.  The details are presented as follows:

As given in section 2, the feature space $\mathcal{X}$ is d-dimensional, and considered in the $k^{th}$ dimension. Since both $X_S$ and $X_T$ can have anomalous samples in that dimension, so we apply a commonly used one-class anomaly detection method, \textbf{One Class SVM} (with outlier parameter $\alpha = 0.1$), on both two domain samples, and compute the mean value of the normal samples in the $k^{th}$ dimension respectively. The mean values are noted as $h_S^{\left( k \right)}$ and $h_T^{\left( k \right)}$. Then we have the value of the dimensional weighted vector in the $k^{th}$ dimension $\tau^{k} \left( X_S, X_T \right)$.
$$\tau^{k} \left( X_S, X_T \right) = |h_S^{\left( k \right)} - h_T^{\left( k \right)}|$$
The dimensional weighted vector for horizontal dimensional weighting is defined by:
\begin{align*}
\tau \left( X_S, X_T \right) &= \left(\tau^{1} \left( X_S, X_T \right) , ..., \tau^{d} \left( X_S, X_T \right)\right)^{T} \\ &= \left( |h_S^{\left( 1 \right)} - h_T^{\left( 1 \right)}|, ..., |h_S^{\left( d \right)} - h_T^{\left( d \right)}| \right)^{T}
\end{align*}
Thus, the normalized dimensional weighted vector is:
$$\tilde{\tau} \left( X_S, X_T \right) = \frac{\tau \left( X_S, X_T \right)}{\max \limits_{k} \left\{ \tau^{k} \left( X_S, X_T \right) \right\}} = \frac{\tau \left( X_S, X_T \right)}{\tau_{max}}$$

\subsection{Analysis}
In this section, the supporting theorems and propositions related to our metric function are given. Then, we present the proof of our probability discrepancy measure. The following theorem shows the relationship between the distribution function and every order moment, which is also the starting point for our metric function.
\begin{theorem}
(Carleman Theorem) For feature space $H$, $X$ is a random variable that obeys distribution $P$ ($X \sim P$). If all raw moments of $X$ exist , that is $\left\{ \mu_{k} = \mathbb{E} \left[ X^{k} \right] | \mu_{k} < \infty; k = 1, 2, 3, ..., n, ...\right\}$, and satisfy Carleman's Condition:
$$\sum_{n = 1}^{\infty} \frac{1}{\left( \mu_{2n}\right)^{\frac{1}{2n}}} = \infty$$
then the distribution $P$ is uniquely determined by all raw moments.
\end{theorem}
The moment generating function of $X$ is defined as $M_X\left( t \right) = \mathbb{E} \left[ e^{tX} \right]$. $M_X\left( t \right) = \mathbb{E} \left[ e^{tX} \right]$ is always valid when $t = 0$, and for nearly all distributions, $M_X\left( t \right) = \mathbb{E} \left[ e^{tX} \right]$ is well-defined within a small neighborhood of $t = 0$, which is $ t \in \left( -\epsilon, \epsilon \right)$.
\begin{proposition}
If $M_X\left( t \right)$ is well-defined on interval $\left( -\epsilon, \epsilon \right)$, then $M_X\left( t \right)$ has every order derivative:
$$M_X^{\left( k \right)} \left( t \right) = \mathbb{E} \left[ X^{k}e^{tX} \right] \Rightarrow M_X^{\left( k \right)} \left( 0 \right) = \mathbb{E} \left[ X^{k} \right]$$
\end{proposition}
\begin{theorem}
If $M_{X} \left( t \right)$ is well-defined on interval $\left( -\epsilon, \epsilon \right)$, then $M_{X} \left( t \right)$ uniquely determines the distribution function of random variable $X$.
\end{theorem}
From the abovementioned theorems and propositions, we determine that if we want to operate on the distribution of a random variable, we only need to operate on its moment generating function. Now, we present a convergence theorem, which reveals a truly important characteristic in the optimization procedure of domain-specific cumulative distribution functions.
\begin{theorem}
(Convergence theorem) If a series of moment generating function $\left\{ M_n\left( t \right) \right\}$ and $\left\{ M_X \left( t \right) \right\}$ is well-defined on a neighborhood of $t = 0$ (denoted as $U \left( 0 \right)$). For every $t \in U \left( 0 \right)$, $\left\{ M_n\left( t \right) \right\}$ converges to $\left\{ M_X \left( t \right) \right\}$, that is: 
$$\lim\limits_{n \rightarrow \infty} M_n \left( t \right) = M_X \left( t \right) \quad \forall t \in U \left( 0 \right)$$
then the variables $X_n$ corresponding to the moment generating function $M_n \left( t \right)$ converge to $X$ in distribution, which is $X_n \stackrel{d}\rightarrow X$.
\end{theorem}
The last theorem we present is the termed expansion theorem, which reveals the motivation for using every order moment to establish our metric function.
\begin{theorem}
(Expansion theorem) If $M_X \left( t \right)$ is well-defined on a neighborhood $U \left( 0 \right)$, then there exist an open interval $\left( - \epsilon, \epsilon \right)$, such that $M_X \left( t \right) < \infty$, $\forall t \in \left( -\epsilon, \epsilon \right)$ and $M_X \left( t \right)$ have series expansion:
$$M_X \left( X \right) = \sum_{n = 0}^{\infty} \frac{\mathbb{E} \left( X^{n} \right)}{n!}t^{n} \quad \forall t \in \left( - \epsilon, \epsilon \right)$$
\end{theorem}
The aforementioned theorems and propositions support the following derivation process of our metric function. Before we directly prove that our DWMD function is strictly a metric, we first prove a lemma.
\begin{lemma}
Consider the space of all sequences S.
$$S \stackrel{\Delta} = \left\{ \xi = \left( \xi_1, \xi_2, ..., \xi_n, ... \right)| \left\{ \xi_n \right\} \text{is an arbitrary sequence} \right\}$$
$\left\{a_n\right\}$ is a sequence that satisfies:
$$\sum_{n = 1}^{\infty} a_n < \infty \quad \text{and} \quad \left\{a_n\right\} \text{monotonically decreases,}$$
then we have the function $F: S \rightarrow \mathbb{R}$, and $\forall \xi \in S$
$$F \left( \xi \right) = \sum\limits_{n = 1}^{\infty}a_n \frac{|\xi|^{\beta}}{C + |\xi|^{\beta}}$$
(where $C$ is a positive constant, $\beta \geq 1$), which has the following properties:
\begin{align*}
&(1)F \left( \xi \right) \geq 0, \forall \xi \in S; F \left( \xi \right) = 0 \Leftrightarrow \xi = 0\\
&(2)F\left( \xi + \omega \right) \leq F \left( \xi \right) + F \left( \omega \right), \forall \xi, \omega \in S\\
&(3)F\left( - \xi \right) = F \left( \xi \right), \forall \xi \in S
\end{align*}
\end{lemma}
\begin{proof}
The properties (1) and (3) are obvious; We mainly focus on property (2). 

For mapping $f$ defined on $\mathbb{R}^{+}$, $f = \frac{\lambda}{C + \lambda}$, the derivative of $f$ is $f^{'} = \frac{C}{\left( C + \lambda \right)^{2}}$, which is nonnegative. Then $f$ is monotonically increasing function on $\mathbb{R}^{+}$. According to the \textbf{\emph{inverse Minkowski inequality}}, we have
$$\frac{|\xi + \omega|^{\beta}}{C + |\xi + \omega|^{\beta}} \leq \frac{|\xi|^{\beta}+|\omega|^{\beta}}{C + |\xi|^{\beta} + |\omega|^{\beta}} \leq \frac{|\xi|^{\beta}}{C + |\xi|^{\beta}} + \frac{|\omega|^{\beta}}{C + |\omega|^{\beta}}$$
That is, $F\left( \xi + \omega \right) \leq F \left( \xi \right) + F \left( \omega \right)$. Therefore lemma is proven.
\end{proof}
Equipped with Lemma 1, the derivation process and proof of our DWMD function are presented as follows.

For the source and target domains $X_S$ and $X_T$, $f_{X_S}\left(x \right)$ and $f_{X_T} \left( x^{*} \right)$ are the density distribution functions of $X_S$ and $X_T$. Their moment generating functions are $M_{X_S} \left( t \right) = \mathbb{E}_{X_S} \left( e^{tx} \right)$ and $M_{X_T} \left( t \right) = \mathbb{E}_{X_T} \left( e^{tx^{*}} \right)$. By the expansion theorem, we have
\begin{align*}
M_{X_S} \left( t\right) &= \sum\limits_{n = 0}^{\infty}\frac{\mathbb{E}\left( X_{S}^{n}\right)}{n !} t^{n}\\ &= < \left( 1, \mathbb{E}X_S^{1}, \mathbb{E}X_{S}^{2}, ..., \mathbb{E}X_{S}^{n}, ... \right), \left( 1, t, t^{2}, ..., t^{n}, ... \right) >
\end{align*}
Similarly, for $M_{X_T} \left( t \right)$
\begin{align*}
M_{X_T} \left( t \right) &= \sum\limits_{n = 0}^{\infty}\frac{\mathbb{E}\left( X_{T}^{n}\right)}{n !} t^{n}\\ &= < \left( 1, \mathbb{E}X_T^{1}, \mathbb{E}X_{T}^{2}, ..., \mathbb{E}X_{T}^{n}, ... \right), \left( 1, t, t^{2}, ..., t^{n}, ... \right) >
\end{align*}
where $< \cdot>$ is the sequence inner product. We can see from the above inner product that when $t$ is fixed, then $M_X \left( t \right)$ is wholly determined by infinite moment sequence $\left( 1, \mathbb{E}X^1, ..., \mathbb{E}X^{n}, ..., \right)$. Construct a mapping 
\begin{align*}
F: \quad & \Omega^{*} \rightarrow S\\
&X \mapsto \left( 1, \mathbb{E}X^1, ..., \mathbb{E}X^n, ... \right)
\end{align*}
where $\Omega^* = \left\{ X \text{ is random variable} | X \stackrel{i.i.d}\sim P; P \in \left\{ X_S, X_T \right\}\right\}$. Let $L = \left\{ \left( 1, \mathbb{E}X^1, ..., \mathbb{E}X^n, ... \right) | X \in \Omega^* \right\}$. It is not hard to prove that $F$ is a one-to-one mapping from $\Omega^*$ to $L$. $\Omega^*$ is noted as $\Omega_S^{*}$ and $\Omega_T^{*}$ when $P = X_S$ and $P = X_T$, respectively. Similarly, when $P = X_S$ and $P = X_T$, then $L$ is written as $L_S$ and $L_T$. Since $e^{\omega}$ satisfies the two requirements in Lemma 1 in each dimension, and according to the lemma, $\mathcal{D}_{new}: L_S \times L_T \rightarrow \mathbb{R}$
$$\mathcal{D}_{new} \left( \mathcal{D}_S || \mathcal{D}_T \right) = \sum_{n = 1}^{\infty} e^{\omega} \odot \frac{| \mathbb{E}X_{S}^{n} - \mathbb{E}X_{T}^{n} |^{\beta}}{C + | \mathbb{E}X_{S}^{n} - \mathbb{E}X_{T}^{n}|^{\beta}}$$
is a metric function between $L_S$ and $L_T$. According to theorem 2, $\mathcal{D}_{new}$ is a strictly defined metric on the source and target distributions.

The reason for choosing $e^{\omega}$ as the coefficient vectors for the series is twofold. First, in each dimension of the coefficient vector, the property of convergence in Lemma 1 is satisfied, which is essential for the derivation process. Second, the $e^{-\frac{\psi x}{\tilde{\tau} \left( X_S, X_T \right)}}$ is a convex function, and for low raw moments, the value of the function is relatively large, while the value decreases sharply when the order $n$ increases, since only finite terms of the series can be calculated in deep network training, and we want the upper error bound to be possibly small. 

The following proposition presents the upper error bound of our metric function when only finite terms of a series are used for model training. 
\begin{proposition}
Let $\mathcal{D}_{train}$ denote the first $n$ terms of $\mathcal{D}_{new}$, $\nu$ is an integer that satisfies $\nu = \lfloor \frac{\psi}{\tau_{max}} \rfloor$. Then the upper error bound is:
$$|\mathcal{D}_{new} - \mathcal{D}_{train}| \leq \sum_{k = n+1}^{\infty} 2^{- \nu k}$$
\end{proposition}
\begin{proof}
Since $C$ is a positive real number, we have:
$$\frac{| \mathbb{E}X_{S}^{n} - \mathbb{E}X_{T}^{n} |^{\beta}}{C + | \mathbb{E}X_{S}^{n} - \mathbb{E}X_{T}^{n}|^{\beta}} < 1$$
So for each order $k$, the latter term in the series can be increased to 1, and we obtain
$$\Rightarrow |\mathcal{D}_{new} - \mathcal{D}_{train}| < \sum_{k = n+1}^{\infty} e^{- \frac{\psi}{\tau_{max} k}}$$
For every $n$, $e^{- \frac{\psi}{\tau_{max} k}} < 2^{- \nu k}$. By the transitivity of inequality, the proposition is proven.
\end{proof}

\section{Experiments}
We evaluate the proposed DWMD on three benchmark datasets for domain adaptation, \emph{Office-31}, \emph{ImageCLEF-DA} and \emph{Office-10} with many state-of-the-art metric functions and other deep learning methods. 

\subsection{Setup}

\textbf{Office-31} is a standard benchmark for domain adaptation from \cite{saenko2010adapting}, comprising 4,652 images and 31 categories collected from three distinct domains: Amazon (A), Webcam (W) and DSLR (D). Following the previous works, we evaluate all methods across all six possible transfer tasks $A \rightarrow W$, $D \rightarrow W$, $W \rightarrow D$, $A \rightarrow D$, $D \rightarrow A$ and $W \rightarrow A$.
\\
\\
\textbf{ImageCLEF-DA}\footnote{http://imageclef.org/2014/adaptation} is also a benchmark dataset for domain adaptation, which contains 12 categories shared by three public datasets, Caltech- 256 (C), ImageNet ILSVRC 2012 (I), and Pascal VOC 2012 (P). Each of the datasets is considered as a domain, comprising 600 images in total and 50 images for every category. We use all the combinations and evaluate methods on six transfer tasks: $I \rightarrow P$, $P \rightarrow I$, $I \rightarrow C$, $C \rightarrow I$, $C \rightarrow P$ and $P \rightarrow C$.
\\
\\
\textbf{Office-10} is a more classic benchmark dataset from \cite{gong2012geodesic}. It is a concise version of \emph{Office-31} and consists of 10 shared categories from three domains: Amazon (A), Webcam (W) and DSLR (D). This dataset is used to compare the effectiveness of different metric functions, i.e. MMD, CMD, DWMD. These metric functions are tested over all six transfer tasks: $A \rightarrow W$, $D \rightarrow W$, $W \rightarrow D$, $A \rightarrow D$, $D \rightarrow A$ and $W \rightarrow A$ with one shared layer for hidden representation matching (Consider only bounded activation function: Sigmoid). Multi-layer cases will be discussed in the analysis section (Consider both bounded and unbounded activation functions: Sigmoid \& ReLU).
\\
\\
For all three datasets, we use the latent representations of Alexnet ~\cite{krizhevsky2012imagenet} and Resnet50 ~\cite{he2016deep}, and train the classifier with one hidden layer and 4,096 and 2,048 hidden nodes, respectively.

We follow the standard training protocol for the datasets in UDA, and we compare our methods with the state-of-the-art hidden representation matching method where the metric function is adopted as one hidden layer domain regularizer in deep networks: central moment discrepancy (CMD). In addition, both shallow and deep domain adaptation methods are considered comparison methods: transfer component analysis (TCA), geodesic flow kernel (GFK), deep domain confusion (DDC), deep adaptation network (DAN), residual transfer network (RTN), and domain adversarial neural network (DANN). All experiments are conducted with randomly shuffled datasets and random initializations.

\begin{table*}
\centering
\caption{Classification accuracy (\%) on \emph{Office-31} dataset for UDA (AlexNet and ResNet50)}
\begin{tabular}{llllllll}
\toprule[2pt]
Method     & \text{   }A$\rightarrow$W &\text{   } D$\rightarrow$W & \text{   }W$\rightarrow$D &\text{   } A$\rightarrow$D &\text{   } D$\rightarrow$A & \text{   }W$\rightarrow$A & Avg  \\
\midrule
AlexNet ~\cite{krizhevsky2012imagenet}    & 61.6 $\pm$ 0.5  & 95.4 $\pm$ 0.3  & 99.0 $\pm$ 0.2  & 63.8 $\pm$ 0.5  & 51.1 $\pm$ 0.6  & 49.8 $\pm$ 0.4  & 70.1 \\
TCA ~\cite{pan2010domain}  & 61.0 $\pm$ 0.0  & 93.2 $\pm$ 0.0  & 95.2 $\pm$ 0.0  & 60.8 $\pm$ 0.0  & 51.6 $\pm$ 0.0  & 50.9 $\pm$ 0.0  & 68.8 \\
GFK  ~\cite{gong2012geodesic}  & 60.4 $\pm$ 0.0  & 95.6 $\pm$ 0.0  & 95.0 $\pm$ 0.0  & 60.6 $\pm$ 0.0  & 52.4 $\pm$ 0.0  & 48.1 $\pm$ 0.0  & 68.7 \\
DDC   ~\cite{tzeng2014deep}  & 61.8 $\pm$ 0.4  & 95.0 $\pm$ 0.5  & 98.5 $\pm$ 0.4  & 64.4 $\pm$ 0.3  & 52.1 $\pm$ 0.6  & 52.2 $\pm$ 0.4  & 70.6 \\
DAN  ~\cite{long2015learning}    & 68.5 $\pm$ 0.5  & 96.0 $\pm$ 0.3  & 99.0 $\pm$ 0.3  & 67.0 $\pm$ 0.4  & 54.0 $\pm$ 0.5  & 53.1 $\pm$ 0.5  & 72.9 \\
RTN    ~\cite{long2016unsupervised}  & \textbf{73.3} $\pm$ 0.3  & \textbf{96.8} $\pm$ 0.2  & \textbf{99.6} $\pm$ 0.1  & 71.0 $\pm$ 0.2  & 50.5 $\pm$ 0.3  & 51.0 $\pm$ 0.1  & 73.7 \\
DANN ~\cite{ganin2017domain}  & 73.0 $\pm$ 0.5  & 96.4 $\pm$ 0.3  & 99.2 $\pm$ 0.3  & \textbf{72.3} $\pm$ 0.3  & 53.4 $\pm$ 0.4  & 51.2 $\pm$ 0.5  & 74.3 \\
CMD ~\cite{zellinger2017central}  &  70.3 $\pm$ 0.4  &  96.0 $\pm$ 0.6       &   99.5 $\pm$ 0.2   &  70.0 $\pm$ 0.1   & 54.5 $\pm$ 0.3   & 53.4 $\pm$ 0.4    &   74.0   \\
DWMD (ours)      &  \textbf{73.3} $\pm$ 0.1  &  \textbf{96.8} $\pm$ 0.2   &  \textbf{99.6} $\pm$ 0.1  &  71.5 $\pm$ 0.4 &  \textbf{57.1} $\pm$ 0.3   &  \textbf{54.8} $\pm$ 0.4   &   \textbf{75.5}   \\
\midrule
ResNet  ~\cite{he2016deep}  & 68.4 $\pm$ 0.2  & 96.7 $\pm$ 0.1  & 99.3 $\pm$ 0.1  & 68.9 $\pm$ 0.2  & 62.5 $\pm$ 0.3  & 60.7 $\pm$ 0.3  & 76.1 \\
TCA  ~\cite{pan2010domain}  & 72.7 $\pm$ 0.0  & 96.7 $\pm$ 0.0  & \textbf{99.6} $\pm$ 0.0  & 74.1 $\pm$ 0.0  & 61.7 $\pm$ 0.0  & 60.9 $\pm$ 0.0  & 77.6 \\
GFK  ~\cite{gong2012geodesic}  & 72.8 $\pm$ 0.0  & 95.0 $\pm$ 0.0  & 98.2 $\pm$ 0.0  & 74.5 $\pm$ 0.0  & 63.4 $\pm$ 0.0  & 61.0 $\pm$ 0.0  & 77.5 \\
DDC  ~\cite{tzeng2014deep}  & 75.6 $\pm$ 0.2  & 96.0 $\pm$ 0.2  & 98.2 $\pm$ 0.1  & 76.5 $\pm$ 0.3  & 62.2 $\pm$ 0.4  & 61.5 $\pm$ 0.5  & 78.3 \\
DAN  ~\cite{long2015learning} & 80.5 $\pm$ 0.4  & 97.1 $\pm$ 0.2  & 99.6 $\pm$ 0.1  & 78.6 $\pm$ 0.2  & 63.6 $\pm$ 0.3  & 62.8 $\pm$ 0.2  & 80.4 \\
RTN  ~\cite{long2016unsupervised}   & \textbf{84.5} $\pm$ 0.2  & 96.8 $\pm$ 0.1  & 99.4 $\pm$ 0.1  & 77.5 $\pm$ 0.3  & 66.2 $\pm$ 0.2  & 64.8 $\pm$ 0.3  & 81.6 \\
DANN  ~\cite{ganin2017domain}  & 82.0 $\pm$ 0.4  & 96.9 $\pm$ 0.2  & 99.1 $\pm$ 0.1  & 79.7 $\pm$ 0.4  & 68.2 $\pm$ 0.4  & \textbf{67.4} $\pm$ 0.5  & 82.2 \\
CMD  ~\cite{zellinger2017central}  &   82.7 $\pm$ 0.6  & 97.2 $\pm$ 0.3 &  99.2 $\pm$ 0.2  &  80.7 $\pm$ 0.5  &  68.6 $\pm$ 0.5  &  65.4 $\pm$ 0.4   &  82.3  \\
DWMD (ours)       &82.1 $\pm$ 0.6 &\textbf{97.9} $\pm$ 0.4 & \textbf{99.6} $\pm$ 0.1  &  \textbf{81.3} $\pm$ 0.6   &  \textbf{69.2} $\pm$ 0.5  &  66.2 $\pm$ 0.1  & \textbf{82.7}  \\   
\bottomrule[2pt]
\end{tabular}
\label{table1}
\end{table*}

\begin{table*}
\centering
\caption{Classification accuracy (\%) on \emph{ImageCLEF-DA} dataset for UDA (AlexNet and ResNet50)}
\begin{tabular}{llllllll}
\toprule[2pt]
Method  & \text{   }I $\rightarrow$ P & \text{   }P $\rightarrow$ I & \text{   }I $\rightarrow$ C & \text{   }C $\rightarrow$ I & \text{   }C $\rightarrow$ P & \text{   }P $\rightarrow$ C & Avg  \\
\midrule
Alexnet ~\cite{krizhevsky2012imagenet} & 66.2 $\pm$ 0.2    & 70.0 $\pm$ 0.2    & 84.3 $\pm$ 0.2    & 71.3 $\pm$ 0.4    & 59.3 $\pm$ 0.5    & 84.5 $\pm$ 0.3    & 73.9 \\
DAN ~\cite{long2015learning}  & 67.3 $\pm$ 0.2    & 80.5 $\pm$ 0.3    & 87.7 $\pm$ 0.3    & 76.0 $\pm$ 0.3    & 61.6 $\pm$ 0.3    & 88.4 $\pm$ 0.2    & 76.9 \\
RTN  ~\cite{long2016unsupervised} & 67.4 $\pm$ 0.3    & 82.3 $\pm$ 0.3    & 89.5 $\pm$ 0.4    & 78.0 $\pm$ 0.2    & 63.0 $\pm$ 0.2    & \textbf{90.1} $\pm$ 0.1    & 78.4 \\
DANN  ~\cite{ganin2017domain}  & 66.5 $\pm$ 0.5    & 81.8 $\pm$ 0.4    & 89.0 $\pm$ 0.5    & \textbf{79.8} $\pm$ 0.5    & 63.5 $\pm$ 0.4    & 88.7 $\pm$ 0.4    & 78.2 \\
CMD  ~\cite{zellinger2017central}  &   67.8 $\pm$ 0.2  &  81.7 $\pm$ 0.2    &   89.7 $\pm$ 0.3   &  77.9 $\pm$ 0.1   &  62.8 $\pm$ 0.6   &  88.3 $\pm$ 0.3   & 78.0  \\
DWMD (ours)   &   \textbf{67.9} $\pm$ 0.3 &   \textbf{82.9} $\pm$ 0.3   &  \textbf{90.1} $\pm$ 0.1  &  77.8 $\pm$ 0.1    &  \textbf{64.0} $\pm$ 0.7  &   89.2 $\pm$ 0.3   &   \textbf{78.7}   \\
\midrule
ResNet ~\cite{he2016deep} & 74.8 $\pm$ 0.3    & 83.9 $\pm$ 0.1    & 91.5 $\pm$ 0.3    & 78.0 $\pm$ 0.2    & 65.5 $\pm$ 0.3    & 91.2 $\pm$ 0.3    & 80.7 \\
DAN ~\cite{long2015learning}  & 75.0 $\pm$ 0.4    & 86.2 $\pm$ 0.2    & 93.3 $\pm$ 0.2    & 84.1 $\pm$ 0.4    & 69.8 $\pm$ 0.4    & 91.3 $\pm$ 0.4    & 83.3 \\
RTN  ~\cite{long2016unsupervised} & 74.6 $\pm$ 0.3    & 85.8 $\pm$ 0.1    & 94.3 $\pm$ 0.1    & 85.9 $\pm$ 0.3    & 71.7 $\pm$ 0.3    & 91.2 $\pm$ 0.4    & 83.9 \\
DANN  ~\cite{ganin2017domain}  & 75.0 $\pm$ 0.6    & 86.0 $\pm$ 0.3    & \textbf{96.2} $\pm$ 0.4    & \textbf{87.0} $\pm$ 0.5    & \textbf{74.3} $\pm$ 0.5    & 91.5 $\pm$ 0.6    & 85.0 \\
CMD  ~\cite{zellinger2017central} & 76.0  $\pm$ 0.5  &  85.4 $\pm$ 0.4   &  94.5 $\pm$  0.4  &  85.5 $\pm$ 0.8  &  72.2 $\pm$ 0.3 & 92.3 $\pm$ 0.5   &   84.3   \\
DWMD (ours)   &  \textbf{76.5} $\pm$ 0.1   &  \textbf{86.4} $\pm$ 0.3   &  95.1 $\pm$ 0.3    &  86.3 $\pm$ 0.4    &  72.7 $\pm$ 0.4  &  \textbf{93.4} $\pm$ 0.4  &   \textbf{85.1}  \\
\bottomrule[2pt]
\end{tabular}
\label{table2}
\end{table*}

\begin{table*}
\centering
\caption{Classification accuracy (\%) on \emph{Office-10} dataset for UDA (AlexNet and ResNet50)}
\begin{tabular}{llllllll}
\toprule[2pt]
Method     & \text{   }A$\rightarrow$W &\text{   } D$\rightarrow$W & \text{   }W$\rightarrow$D &\text{   } A$\rightarrow$D &\text{   } D$\rightarrow$A & \text{   }W$\rightarrow$A & Avg  \\
\midrule
AlexNet ~\cite{krizhevsky2012imagenet}    & 66.3 $\pm$ 0.5  & 96.6 $\pm$ 0.2  & 96.1 $\pm$ 0.2  & 82.1 $\pm$ 0.3  & 73.4 $\pm$ 0.3  & 64.1 $\pm$ 0.2  & 79.8 \\
MMD ~\cite{gretton2007kernel}  &  66.4 $\pm$ 0.1  &  96.2 $\pm$ 0.1       &   96.8 $\pm$ 0.4   &  82.8 $\pm$ 0.4   & 73.6 $\pm$ 0.3   & 65.4 $\pm$ 0.3    &   80.2   \\
CMD ~\cite{zellinger2017central}  &  81.0 $\pm$ 0.1  &  96.6 $\pm$ 0.1       &   \textbf{99.4} $\pm$ 0.0   &  89.8 $\pm$ 0.0   & 77.6 $\pm$ 0.3   & \textbf{76.2} $\pm$ 0.1    &   86.8   \\
DWMD (ours)      &  \textbf{85.1} $\pm$ 0.1  &  \textbf{97.0} $\pm$ 0.0   &  \textbf{99.4} $\pm$ 0.0  &  \textbf{91.7} $\pm$ 0.0 &  \textbf{77.9} $\pm$ 0.4   &  75.6 $\pm$ 0.2   &   \textbf{87.8}   \\
\midrule
ResNet  ~\cite{he2016deep}  & 68.4 $\pm$ 0.2  & 96.7 $\pm$ 0.1  & 99.3 $\pm$ 0.1  & 68.9 $\pm$ 0.2  & 62.5 $\pm$ 0.3  & 60.7 $\pm$ 0.3  & 76.1 \\
MMD ~\cite{gretton2007kernel}  & 82.0 $\pm$ 0.4  & 96.9 $\pm$ 0.2  & 99.1 $\pm$ 0.1  & 79.7 $\pm$ 0.4  & 68.2 $\pm$ 0.4  & \textbf{67.4} $\pm$ 0.5  & 82.2 \\
CMD  ~\cite{zellinger2017central}  &   96.6 $\pm$ 0.0  & 96.1 $\pm$ 0.1 &  100 $\pm$ 0.0  &  96.1 $\pm$ 0.6  &  94.2 $\pm$ 0.1  &  94.1 $\pm$ 0.0   &  96.2  \\
DWMD (ours)       & \textbf{97.3} $\pm$ 0.1 &\textbf{97.2} $\pm$ 0.1 & \textbf{100} $\pm$ 0.0  &  \textbf{96.3} $\pm$ 0.0   &  \textbf{94.3} $\pm$ 0.0  &  \textbf{94.6} $\pm$ 0.1  & \textbf{96.6}  \\   
\bottomrule[2pt]
\end{tabular}
\label{table3}
\end{table*}

\begin{figure*} 
\subfigure[CMD: \emph{Source}=A] { \label{fig:1} 
\includegraphics[width=0.48\columnwidth]{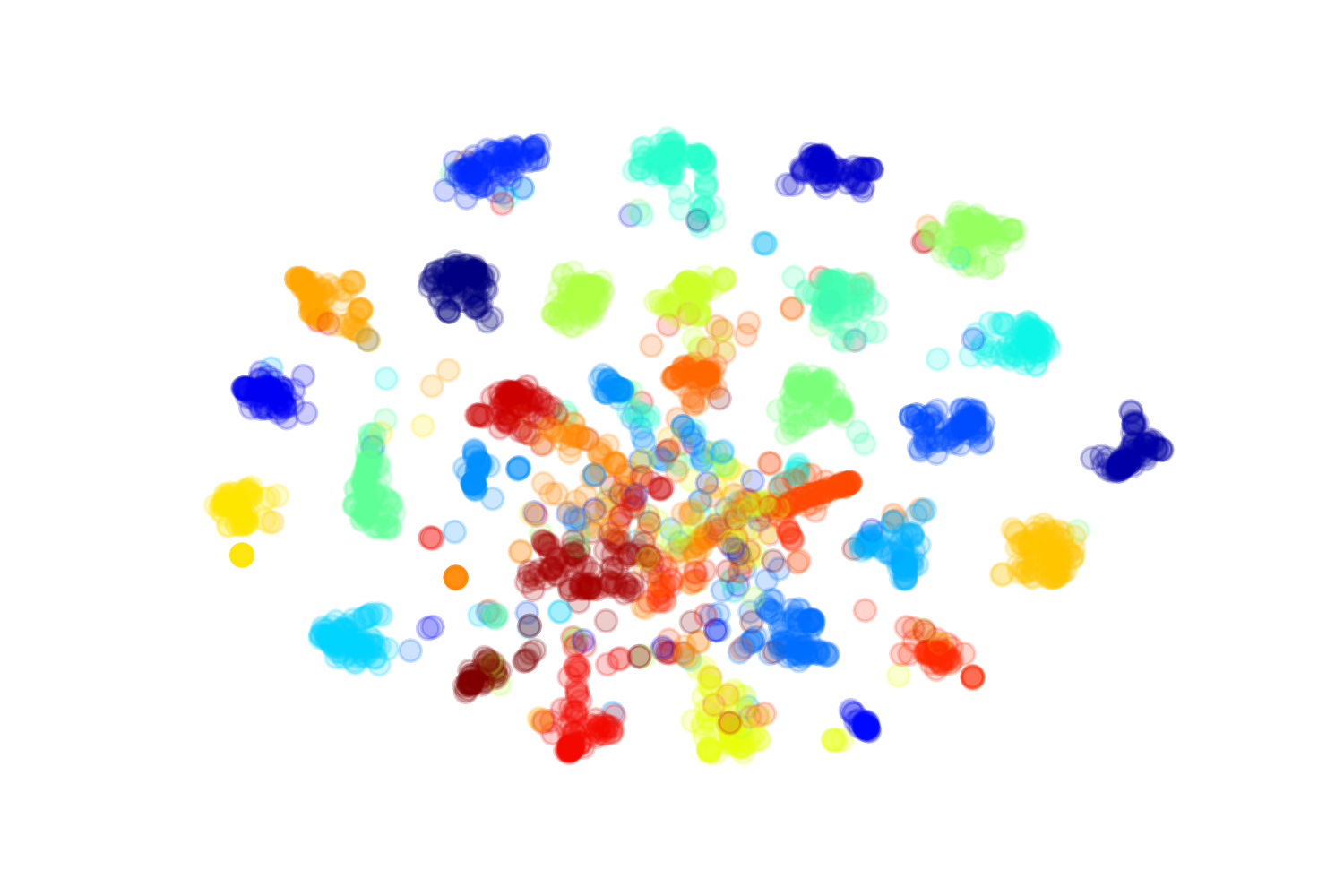} 
} 
\subfigure[CMD: \emph{Target}=D] { \label{fig:2} 
\includegraphics[width=0.48\columnwidth]{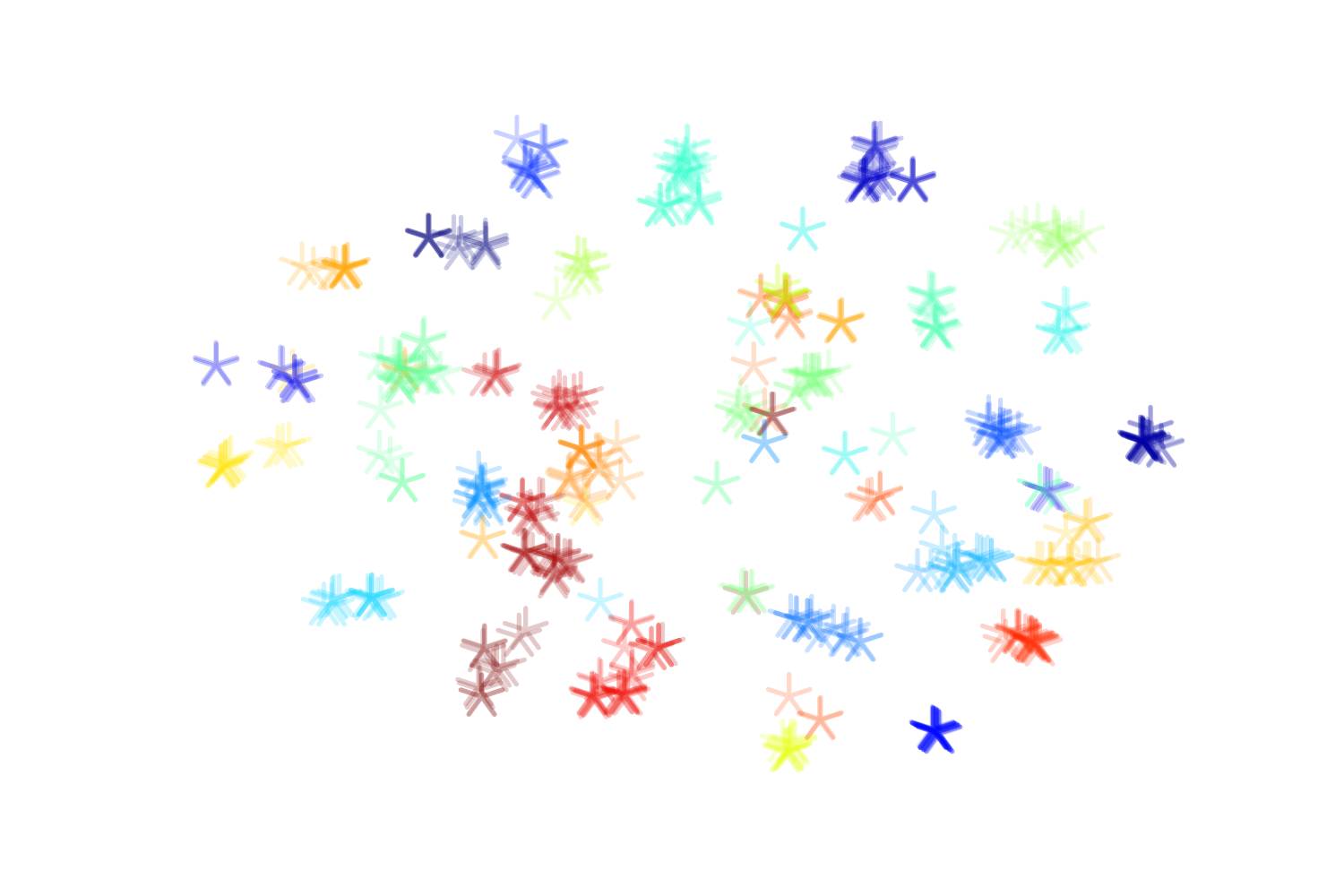} 
}
\subfigure[DWMD: \emph{Source}=A] { \label{fig:3} 
\includegraphics[width=0.48\columnwidth]{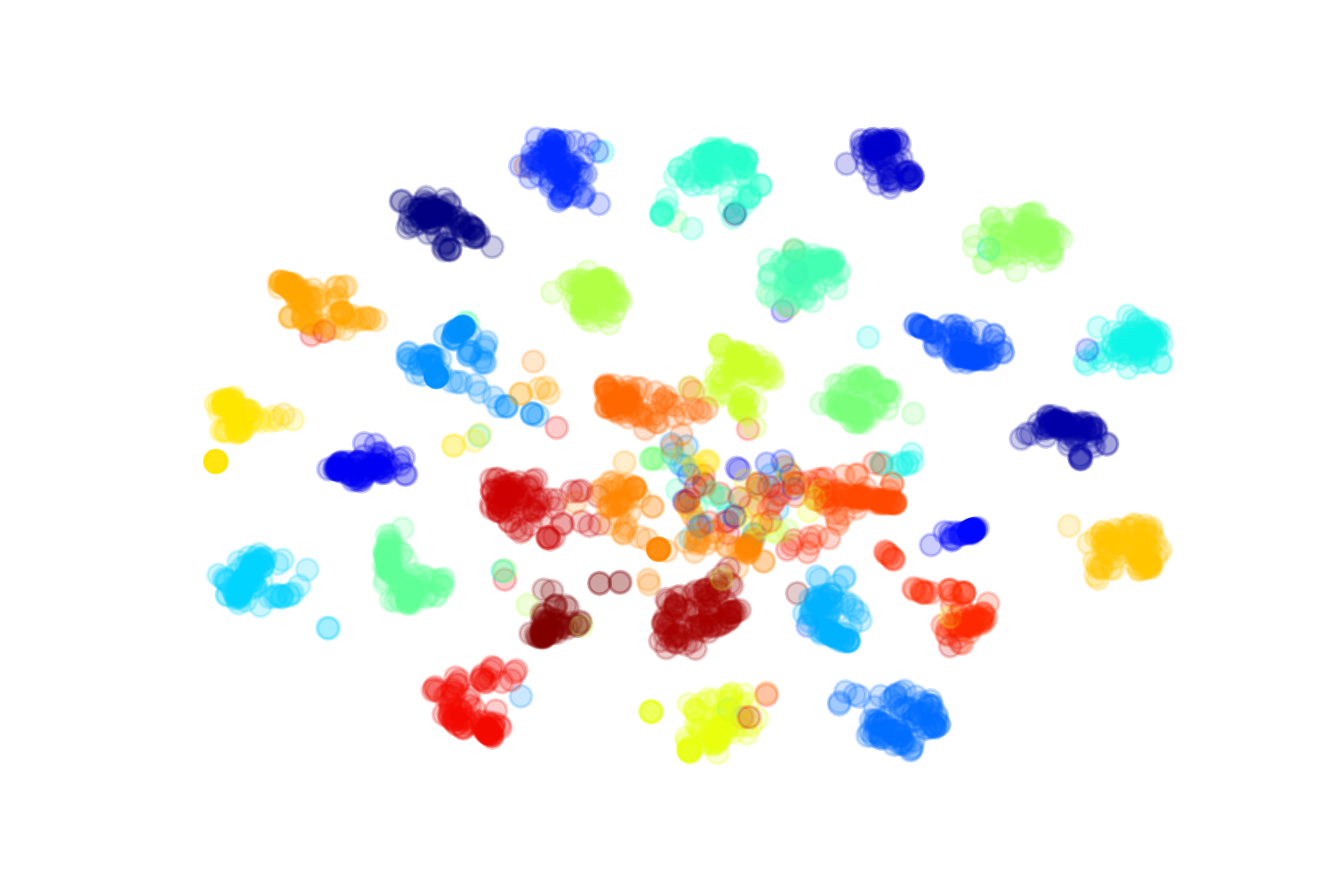} 
}
\subfigure[DWMD: \emph{Target}=D] { \label{fig:4} 
\includegraphics[width=0.48\columnwidth]{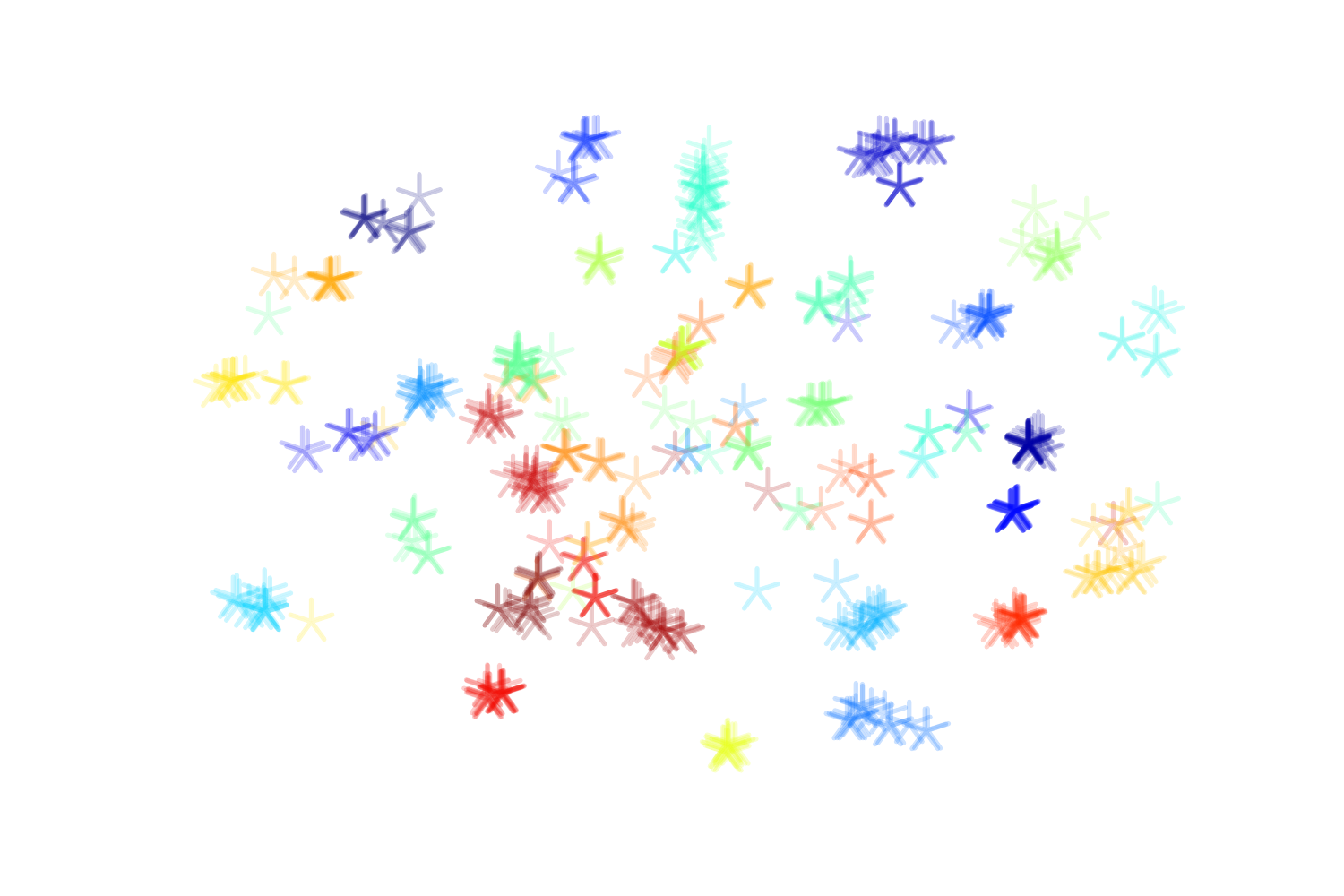} 
} 
\caption{The t-SNE visualization of hidden representations (ResNet50) generated by CMD (1)(2) and DWMD (3)(4) on \emph{Office-31}} 
\label{fig:t-SNE} 
\end{figure*}

\subsection{Results}
The classification results on the $\emph{Office-31}$ dataset for UDA are shown in Table ~\ref{table1} (Penalty parameter $\lambda = 1$, Positive constant $C = 0.05$,  $\beta = 1$, and moment order $n= 5$). The base neural networks chosen for the experiments are AlexNet and ResNet50. The proposed DWMD method outperforms all comparison methods on most transfer tasks, and reaches a total average classification accuracy of 75.5\% (AlexNet) and 82.7\% (ResNet) with only one hidden layer for representation matching. Specifically, our DWMD method promotes the classification results substantially in some hard transfer tasks, such as $A \rightarrow D$, $D \rightarrow A$ and $W \rightarrow A$, where the source and target domains are quite different, and also achieves higher performance on easy transfer tasks, e.g. $D \rightarrow W$ and $W \rightarrow D$. From the above results, we can make the observation that even one hidden layer representation matching model with the DWMD domain regularizer can be effective.
\\
\\
The results on the more domain size balanced dataset $\emph{ImageCLEF-DA}$ are reported in Table ~\ref{table2} (Penalty parameter $\lambda = 1$, Positive constant $C = 0.1$,  $\beta = 1$, and moment order $n= 5$). The DWMD method substantially outperforms other compared approaches on most transfer tasks. However, we can see that smaller improvements are presented. Our interpretation for this is that since the dataset $\emph{ImageCLEF-DA}$ is visually more similar among categories than $\emph{Office-31}$, less shift is generated during the transfer tasks, which alleviates difficulties in domain adaptation.
\\
\\
The unsupervised domain adaptation results on \emph{Office-10} transfer tasks are shown in Table ~\ref{table3} (Penalty parameter $\lambda = 1$, Positive constant $C = 0.05$,  $\beta = 1$, and moment order $n= 5$). This group of experiments aims to directly compare the effectiveness of different distribution metrics. We can observe that DWMD outperforms the comparison metrics on nearly all transfer tasks, and achieves significant improvement in two hard transfer tasks: $A \rightarrow W$ and $A \rightarrow D$. From the results, the observation can be made that, in hidden representation matching models, DWMD metric is more accurate in describing distribution discrepancy information than CMD, MMD metrics.

\section{Analysis}
\subsection{Feature Visualization}
To better illustrate the transferability of DWMD method, we visualize the hidden representations learned by CMD and DWMD respectively on transfer tasks $A$ and $D$ (\emph{Office-31}) in figure ~\ref{fig:t-SNE} using t-SNE embeddings ~\cite{donahue2014decaf}. According to the results, we can make the observation that DWMD is more powerful than CMD in UDA scenario. 

\subsection{Analysis of Parameter Sensitivity}
In this subsection, we investigate the hyperparameters $C$, $\beta$ and the moment order $n$ in the DWMD function. For the experiments on hyperparameters $C$ and $\beta$, we use \emph{Office-31} as testing dataset and the order of moment is fixed to 5. Figure~\ref{fig:C} demonstrates the total average transfer accuracy based on both AlexNet and ResNet50 with $\beta = 1$ by varying $C \in \left\{0.01, 0.03, 0.05, 0.07, 0.1, 0.5, 1.0 \right\}$. The two curves in the line chart are bell-shaped as the accuracies remain stable at first and then decrease slightly when $C$ gets bigger, which confirms that the hyperparameter $C$ in our metric has stability. As for hyper-parameter $\beta$, Figure~\ref{fig:Beta} shows the specific results based on ResNet50 with $C = 0.05$ on hard transfer tasks $A \rightarrow W$, $A \rightarrow D$, $D \rightarrow A$, and $W \rightarrow A$ by having $\beta \in \left\{0.5, 0.8, 1 \right\}$. The closeness and little fluctuations with respect to accuracies for specific transfer tasks in this bar chart reveal the stability of hyperparameter $\beta$ in our function.

\begin{figure*} 
\subfigure[Stability Analysis on Hyperparameter C] { \label{fig:C} 
\includegraphics[width=0.5\columnwidth]{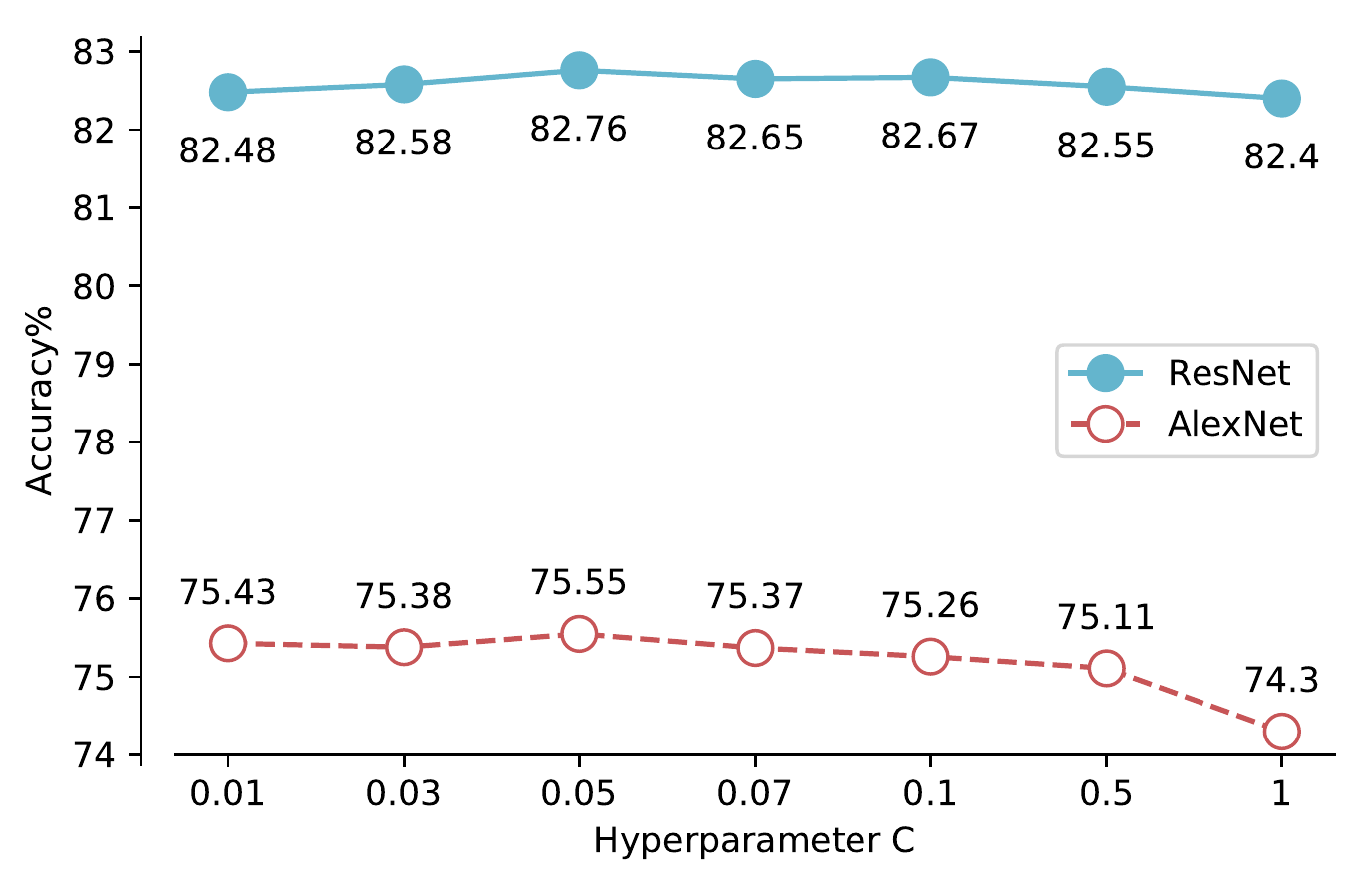} 
} 
\subfigure[Stability Analysis on Hyperparameter $\beta$] { \label{fig:Beta} 
\includegraphics[width=0.5\columnwidth]{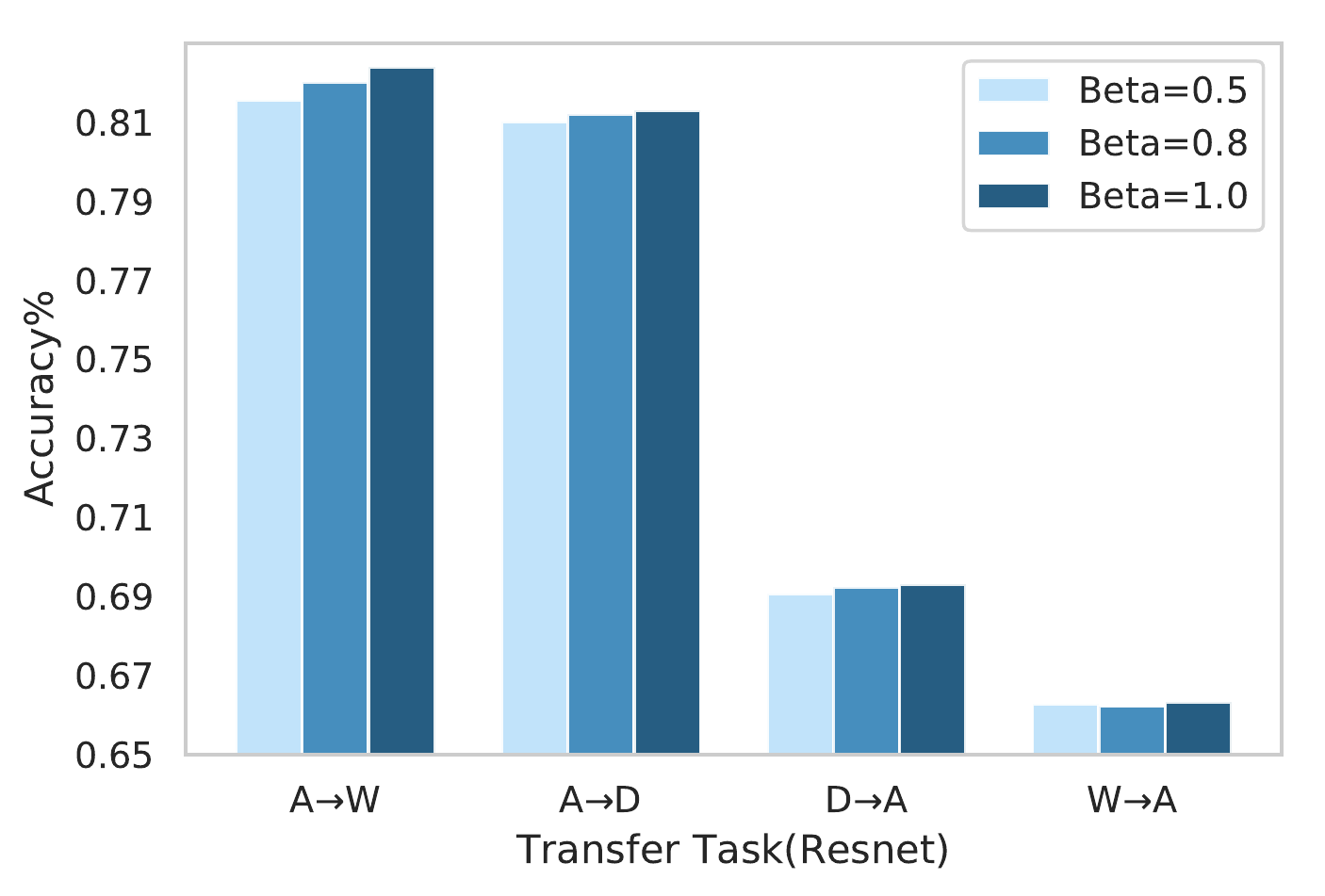} 
}  
\end{figure*}

\begin{table*}
\centering
\caption{Classification accuracy (\%) on \emph{Office-31} dataset applying DWMD with different moment order $n$ (AlexNet and ResNet50)}
\begin{tabular}{clllllll}
\toprule[2pt]
order\text{   }$n$     & \text{   }A$\rightarrow$W &\text{   } D$\rightarrow$W & \text{   }W$\rightarrow$D &\text{   } A$\rightarrow$D &\text{   } D$\rightarrow$A & \text{   }W$\rightarrow$A & Avg  \\
\midrule
$n = 2 / AlexNet$   & 64.8 $\pm$ 0.3  & 95.7 $\pm$ 0.4  & 99.0 $\pm$ 0.3  & 67.8 $\pm$ 0.2  & 52.9 $\pm$ 0.1  & 51.0 $\pm$ 0.2  & 71.9 \\
$n = 3 / AlexNet$  &  67.2 $\pm$ 0.6  &  96.2 $\pm$ 0.6       &   99.2 $\pm$ 0.3   &  68.0 $\pm$ 0.7   & 54.5 $\pm$ 0.6   & 52.6 $\pm$ 0.1    &   73.0   \\
$n = 5 / AlexNet$      &  \textbf{73.3} $\pm$ 0.1  &  \textbf{96.8} $\pm$ 0.2   &  99.6 $\pm$ 0.1  & \textbf{71.5} $\pm$ 0.4 &  \textbf{57.1} $\pm$ 0.3   & \textbf{54.8} $\pm$ 0.4   &   \textbf{75.5}   \\
$n = 10 / AlexNet$      &  72.2 $\pm$ 0.4  &  96.4 $\pm$ 0.0   & \textbf{99.8} $\pm$ 0.0  &  70.4 $\pm$ 0.6 &  56.2 $\pm$ 0.4   &  54.3 $\pm$ 0.3   &   74.9   \\
$n = 20 / AlexNet$      &  72.6 $\pm$ 0.6  &  96.7 $\pm$ 0.6   &  \textbf{99.8} $\pm$ 0.2  &  70.1 $\pm$ 0.8 &  56.0 $\pm$ 0.5   &  54.0 $\pm$ 0.5   &   74.9   \\
\midrule
$n = 2 / ResNet$   & 80.4 $\pm$ 0.1  & 97.2 $\pm$ 0.0  & 99.6 $\pm$ 0.1  & 80.7 $\pm$ 0.4  & 65.2 $\pm$ 0.2  & 62.9 $\pm$ 0.2  & 81.0 \\
$n = 3 / ResNet$  &  81.4 $\pm$ 0.0  &  97.3 $\pm$ 0.0       &   99.6 $\pm$ 0.0   &  81.1 $\pm$ 0.3   & 66.4 $\pm$ 0.5   & 64.2 $\pm$ 0.3    &   81.7   \\
$n = 5 / ResNet$  &  82.1 $\pm$ 0.6  &  \textbf{97.9} $\pm$ 0.1   &  99.6 $\pm$ 0.1 &  \textbf{81.3} $\pm$ 0.6 &  \textbf{69.2} $\pm$ 0.5   &  \textbf{66.2} $\pm$ 0.1  &   \textbf{82.7}   \\
$n = 10 / ResNet$      &  \textbf{82.9} $\pm$ 0.6  &  97.6 $\pm$ 0.3   &  99.6 $\pm$ 0.0  &  81.1 $\pm$ 0.9 &  67.6 $\pm$ 0.6   &  65.9 $\pm$ 0.6   &   82.5   \\
$n = 20 / ResNet$      &  81.0 $\pm$ 0.3  &  97.6 $\pm$ 0.2   &  \textbf{99.8} $\pm$ 0.0  &  80.9 $\pm$ 0.9 &  67.0 $\pm$ 0.4   &  66.0 $\pm$ 0.4   &   82.1   \\
\bottomrule[2pt]
\end{tabular}
\label{table4}
\end{table*}

\begin{figure*} 
\subfigure[Results on \emph{ImageCLEF-DA} Dataset] { \label{fig:a} 
\includegraphics[width=0.5\columnwidth]{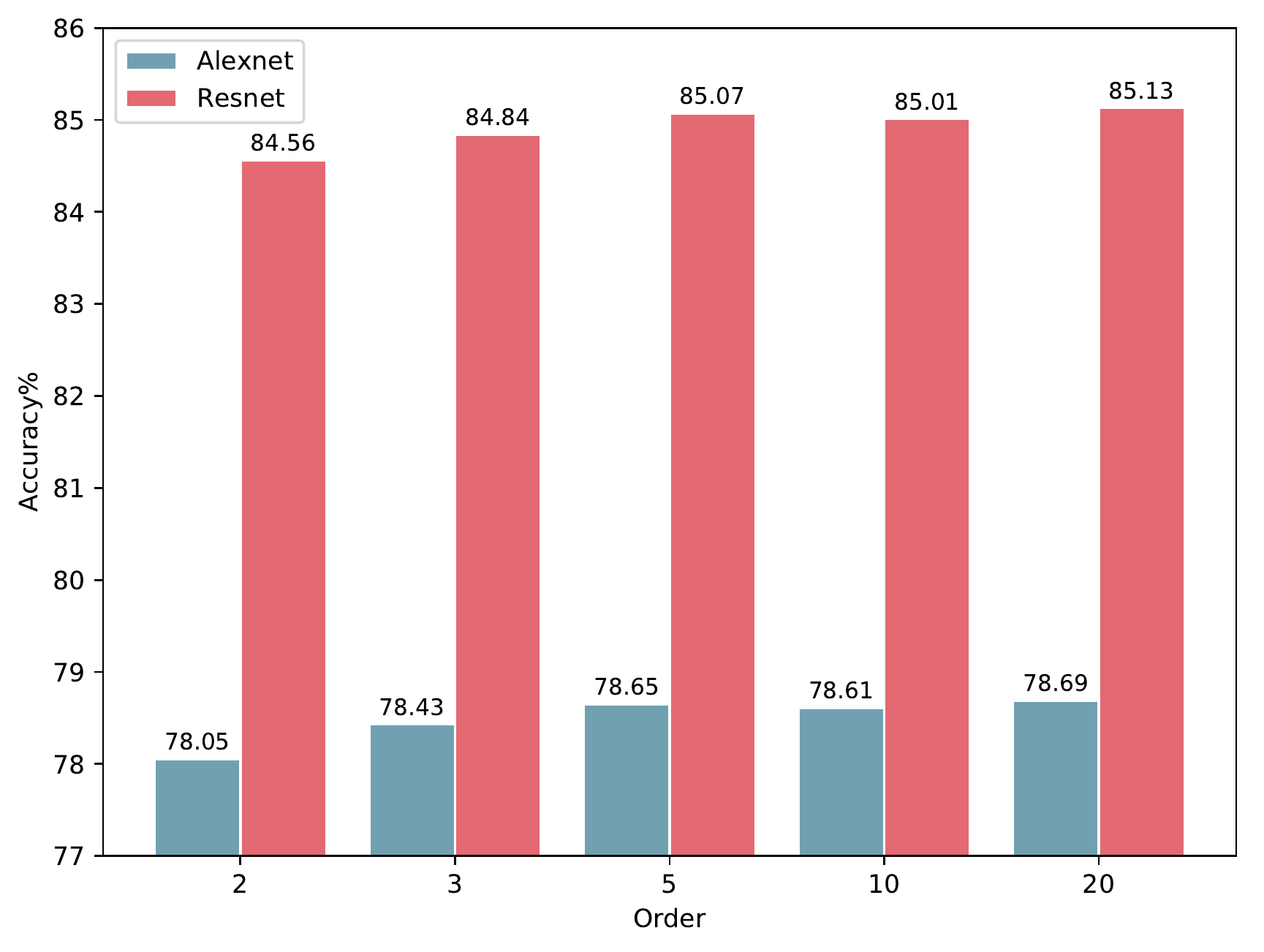} 
} 
\subfigure[Results on \emph{Office-10} Dataset] { \label{fig:b} 
\includegraphics[width=0.5\columnwidth]{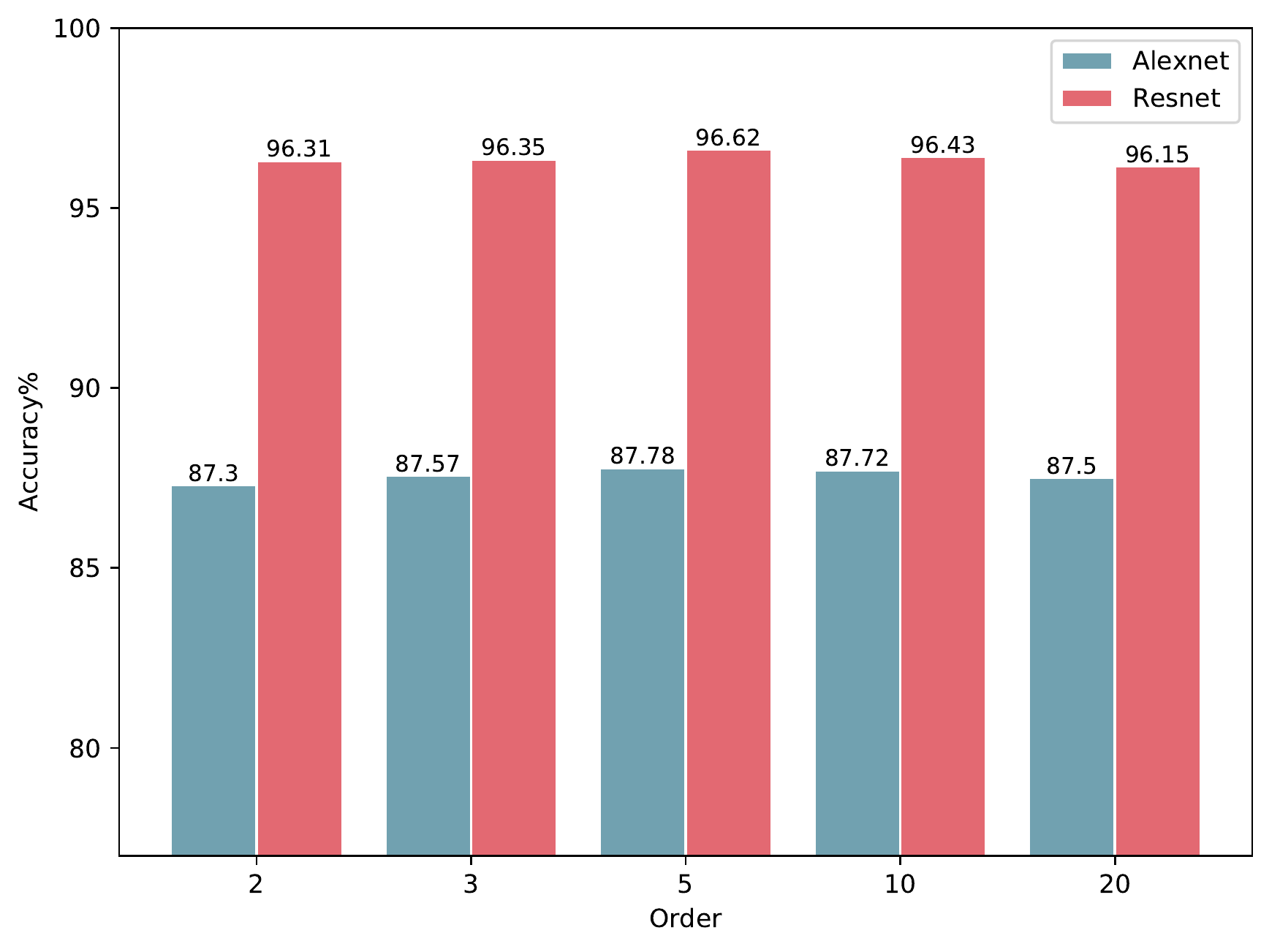} 
} 
\caption{Classification accuracy (\%) on two datasets applying DWMD with different moment order $n$ (AlexNet and ResNet50)} 
\label{fig:moment_order} 
\end{figure*}

Although obtaining an accurate form of DWMD requires $n$ goes to $+\infty$, we set $n$ to a finite positive integer in practice. Comprehensive experiments are carried out to check the sensitivity of moment order $n$ by fixing $\beta = 1$ and $C$ ($C = 0.05$ for \emph{Office-31} and \emph{Office-10}; $C = 0.1$ for \emph{ImageCLEF-DA}). The classification accuracy results on three datasets for UDA based on AlexNet and ResNet are shown in Table ~\ref{table4} and Figure ~\ref{fig:moment_order}. $\left\{ 2, 3, 4, 10, 20 \right\}$ is a set of all testing values for moment order $n$. As revealed from the graphs that the classification accuracy will significantly increase if we set a larger moment order $n$ in the beginning. This encouraging results highlight the importance of using high order moment alignment in UDA. Another surprising fact we found is that we actually need not to set $n$ to a sufficient large positive integer. Instead, assigning a small positive integer like 5 to the moment order $n$ is enough. We provide our explanation: When $n$ increases, the scaleable term $e^{\omega}$ will shrink to 0 at exponential speed. Thus, if $n$ is sufficiently large (for example, $n$ is larger than $10$), the high order moment term in the DWMD function will actually have minimal impact for hidden representation matching. Also, the sample noise might mislead the training direction when the series have two many terms. Therefore, according to the experimental results, we recommend to set order moment $n$ to a positive integer between $5$ and $10$. This can free us from the worry that training an UDA model with DWMD might be time-consuming. Beside, for the results on \emph{Office-31}, we can see that when we increase the moment order $n$ from $2$ to $5$, the average classification accuracy improves $3.6\%$ when choosing AlexNet as the feature extraction model. However, we get a much smaller accuracy improvement of $1.7\%$
when it comes to ResNet50. The similar phenomenon appear in hard transfer tasks: $A \rightarrow W$, $W \rightarrow A$, $A \rightarrow D$, and $D \rightarrow A$. These results lead to an insightful observation: the DWMD function can be really effective in UDA even when the feature representations obtained are not good enough.

\begin{table*}
\centering
\caption{Comparison between DWMD and SMD}
\begin{tabular}{|l|l|l|l|l|l|l|}
\hline
\multirow{3}{*}{$\qquad$\textbf{Method}} & \multicolumn{6}{l|}{$\qquad \qquad \qquad \qquad $\textbf{Average Accuracy} (\%)}                                                                                                                                       \\ \cline{2-7} 
                        & \multicolumn{2}{l|}{$\qquad$\textbf{\emph{Office-31}}} & \multicolumn{2}{l|}{$\quad$\textbf{\emph{ImageCLEF-DA}}} & \multicolumn{2}{l|}{$\qquad$\textbf{\emph{Office-10}}} \\ \cline{2-7} 
                        & AlexNet                   & ResNet50                   & AlexNet                     & ResNet50                    & AlexNet                   & ResNet50                   \\ \hline
SMD (order=2)           & \text{  }\text{  }71.68                     & $\quad$80.92                      &  \text{  }\text{  }77.89                       & $\quad$84.50                       & \text{  }\text{  }87.13                     & $\quad$96.25                      \\ \hline
DWMD (order=2)          & \text{  }\text{  }71.87                     & $\quad$81.00                      & \text{  }\text{  }78.05                       & $\quad$84.56                       &  \text{  }\text{  }87.30                     & $\quad$96.31                      \\ \hline
SMD (order=5)           & \text{  }\text{  }75.15                     & $\quad$82.56                      & \text{  }\text{  }78.57                       & $\quad$85.01                       &  \text{  }\text{  }87.53                     & $\quad$96.57                      \\ \hline
DWMD (order=5)          & \text{  }\text{  }75.52                     & $\quad$82.72                      & \text{  }\text{  }78.65                       & $\quad$85.07                       &  \text{  }\text{  }87.78                     & $\quad$96.62                      \\ \hline
SMD (order=10)          & \text{  }\text{  }74.42                     & $\quad$82.41                      & \text{  }\text{  }78.54                       & $\quad$85.05                       &  \text{  }\text{  }87.52                     & $\quad$96.41                      \\ \hline
DWMD (order=10)         & \text{  }\text{  }74.88                     & $\quad$82.45                      & \text{  }\text{  }78.61                       & $\quad$85.01                       &  \text{  }\text{  }87.72                     & $\quad$96.43                      \\ \hline
\end{tabular}
\label{table5}
\end{table*}

\subsection{The effectiveness of Dimensional Weights}
We demonstrate the effectiveness of dimensional weights by introducing a new metric termed Scaleable Orderwise Moment Discrepancy (SMD), which is similar to DWMD but with dimensional weighted vector removed. In SMD metric, we replace the dimensional weighted vector ($\tilde \tau \left( X_{S}, X_{T} \right)$) with a constant vector $\tau_{c}$ and the constant is defined as the average of the sum of each vector component, i.e. $\tau_{c} = \frac{1}{d} \sum \limits_{i = 1}^{d} \tilde{\tau}^{i} \left( X_S, X_T \right)$. The comparisons between DWMD and SMD are provided in Table ~\ref{table5}. Results indicate that dimensional weighted vector can better align those feature dimensions that have larger discrepancy. Thus, introducing dimensional weighted are more effective in bridging cross-domain discrepancy.

\begin{table}
\centering
\caption{Representation Matching Network Structure}
\begin{tabular}{|c|l|l|l|l|l|l|}
\hline
\multicolumn{7}{|c|}{Number of Hidden Representation Matching Layer (HRML) \& Activation Function} \\ \hline
\multicolumn{7}{|c|}{2 HRMLs: 1 layer (1024 nodes, ReLU) + 1 layer (256 nodes, Sigmoid)}           \\ \hline
\multicolumn{7}{|c|}{3 HRMLs: 2 layers (1024, 512 nodes, ReLU) + 1 layer (256 nodes, Sigmoid)}     \\ \hline
\multicolumn{7}{|c|}{4 HRMLs: 3layers (1024, 512, 256 nodes, ReLU) + 1 Layer (256 nodes, Sigmoid)} \\ \hline
\end{tabular}
\label{table6}
\end{table}

\begin{table}
\centering
\caption{Classification accuracy (\%) with unbounded activation functions (Base model: ResNet50)}
\begin{tabular}{clllllll}
\toprule[2pt]
\multicolumn{8}{c}{\emph{Office-31} Dataset}\\
\midrule[2pt]
Methods     & \text{   }A$\rightarrow$W &\text{   } D$\rightarrow$W & \text{   }W$\rightarrow$D &\text{   } A$\rightarrow$D &\text{   } D$\rightarrow$A & \text{   }W$\rightarrow$A & Avg  \\
\midrule
CMD (2 HRMLs)   &  81.6 $\pm$ 0.3  & 97.0 $\pm$ 0.2  & 99.2 $\pm$ 0.0 & 80.7 $\pm$ 0.3  & 67.3 $\pm$ 0.3 & 65.2 $\pm$ 0.5 & 81.8 \\
DWMD (2 HRMLs) & \textbf{82.6} $\pm$ 0.1 & \textbf{97.9} $\pm$ 0.3 & \textbf{99.6} $\pm$ 0.0 & \textbf{82.0} $\pm$ 0.1 & \textbf{67.4} $\pm$ 0.5 & \textbf{65.3} $\pm$ 0.0 & \textbf{82.5} \\
\midrule
CMD (3 HRMLs) & 82.0 $\pm$ 0.1 & 97.4 $\pm$ 0.2 & 99.4 $\pm$ 0.0 & 82.1 $\pm$ 0.1 & 66.9 $\pm$ 0.2 & 64.3 $\pm$ 0.4 & 82.0 \\
DWMD (3 HRMLs) & \textbf{83.9} $\pm$ 0.2 & \textbf{97.7} $\pm$ 0.1 & \textbf{99.6} $\pm$ 0.0 & \textbf{83.1} $\pm$ 0.2 & \textbf{68.8} $\pm$ 0.3 & \textbf{67.0} $\pm$ 0.3 & \textbf{83.4} \\
\midrule
CMD (4 HRMLs) & 82.8 $\pm$ 0.0 & 96.2 $\pm$ 0.2 & 99.2 $\pm$ 0.0 & 81.3 $\pm$ 0.5 & 66.9 $\pm$ 0.4 & 63.8 $\pm$ 0.4 & 81.7 \\
DWMD (4 HRMLs) & \textbf{85.6} $\pm$ 0.2 & \textbf{97.4} $\pm$ 0.3 & \textbf{99.6} $\pm$ 0.0 & \textbf{82.6} $\pm$ 0.6& \textbf{68.3} $\pm$ 0.6 & \textbf{66.6} $\pm$ 0.4 & \textbf{83.4} \\
\midrule[2pt]
\multicolumn{8}{c}{\emph{ImageCLEF-DA} Dataset}\\
\midrule[2pt]
Method  & \text{   }I $\rightarrow$ P & \text{   }P $\rightarrow$ I & \text{   }I $\rightarrow$ C & \text{   }C $\rightarrow$ I & \text{   }C $\rightarrow$ P & \text{   }P $\rightarrow$ C & Avg  \\
\midrule
CMD (2 HRMLs) & 75.7 $\pm$ 0.1 & 86.9 $\pm$ 0.1 & 94.6 $\pm$ 0.0 & 85.1 $\pm$ 0.3 & \textbf{72.5} $\pm$ 0.4 & 94.0 $\pm$ 0.1 & 84.8 \\
DWMD (2 HRMLs) & \textbf{77.2} $\pm$ 0.1 & \textbf{87.4} $\pm$ 0.2 & \textbf{95.3} $\pm$ 0.0 & \textbf{86.1} $\pm$ 0.1 & 72.1 $\pm$ 0.3 & \textbf{94.7} $\pm$ 0.2 & \textbf{85.5} \\
\midrule
CMD (3 HRMLs) & 76.2 $\pm$ 0.0 & 86.8 $\pm$ 0.2 & 94.1 $\pm$ 0.0 & 86.1 $\pm$ 0.1 & 72.6 $\pm$ 0.1 & 93.2 $\pm$ 0.0 & 84.8 \\
DWMD (3 HRMLs) & \textbf{76.4} $\pm$ 0.0 & \textbf{87.9} $\pm$ 0.1 & \textbf{95.9} $\pm$ 0.1 & \textbf{88.2} $\pm$ 0.0 & \textbf{73.3} $\pm$ 0.2 & \textbf{95.4} $\pm$ 0.0 & \textbf{86.2} \\
\midrule
CMD (4 HRMLs) & 75.5 $\pm$ 0.5 & 86.0 $\pm$ 0.5 & 94.7 $\pm$ 0.3 & 85.4 $\pm$ 0.9 & 73.1 $\pm$ 0.3 & 93.4 $\pm$ 0.1 & 84.7 \\
DWMD (4 HRMLs) & \textbf{77.2} $\pm$ 0.3 & \textbf{88.1} $\pm$ 0.5 & \textbf{95.3} $\pm$ 0.4 & \textbf{88.4} $\pm$ 0.5 & \textbf{74.0} $\pm$ 0.4 & \textbf{95.0} $\pm$ 0.2 & \textbf{86.3} \\
\bottomrule[2pt]
\end{tabular}
\label{table7}
\end{table}

\subsection{Generalize to Unbounded Activation Functions}
As mentioned before, our DWMD metric improves the CMD by alleviating its compact interval data distribution assumption. We examine this fact on the basis of the metric performance on rebuilt multi hidden representation matching layers (HRMLs) with unbounded activation functions. We experiment two metrics with different numbers of HRMLs. The details of each representation matching neural network are provided in Table ~\ref{table6}. Since ReLU activation function is introduced in the representation matching network, the assumption that hidden activations are bounded might be violated. This claim is testified by the results reported in Table~\ref{table7} (Penalty parameter $\lambda = 1$, Positive constant $C = 0.05$ for \emph{Office-31}, $C = 0.1$ for \emph{ImageCLEF-DA},  $\beta = 1$, and moment order $n= 5$). When we apply unbounded activation function ReLU, the DWMD method will significantly outperform the CMD metric. As the number of HRMLs increases, the compact interval data distribution assumption of the CMD method is more likely to fail, which result in larger gaps in accuracy comparing with the DWMD method. This reveals the superiority of the DWMD metric over the CMD metric.

\section{Conclusion}
This paper proposed a novel moment-based probability distribution metric termed dimensional weighted orderwise moment discrepancy (DWMD) for feature representation matching in a UDA scenario. Unlike previous matching methods, all order moments are explicitly aligned vertically in our metric function, taking the form of a series, and the dimensional weighted vector reflecting discrepancy in each dimension horizontally is considered. Furthermore, we also compute the error bound of our metric when using its empirical estimate. Comprehensive experiments on benchmark datasets demonstrated the efficacy of the proposed approach. 

\section{Acknowledgment}
This research was partially supported by "The Fundamental Theory and Applications of Big Data with Knowledge Engineering" under the National Key Research and Development Program of China with Grant No. 2018YFB1004500, the MOE Innovation Research Team No. IRT17R86, the National Science Foundation of China under Grant Nos. 61721002 and 61532015, and Project of SERVYOU-XJTU Joint Innovation Center of Big Tax Data. I would like to thank Professor Limin Li for her constructive comments on this paper.

\bibliographystyle{unsrt}  
\bibliography{template}  

\end{document}